\documentclass[11pt]{article}
\usepackage[preprint]{acl}

\usepackage[T1]{fontenc}
\usepackage[utf8]{inputenc}
\usepackage{times}
\usepackage{latexsym}
\usepackage{microtype}
\usepackage{booktabs}
\usepackage{amsmath,amssymb,amsthm}
\usepackage{graphicx}
\usepackage{tikz}
\usepackage{pgfplots}
\pgfplotsset{compat=1.18}
\usepackage{xcolor}
\usepackage{url}
\usepackage{seqsplit}

\usepackage{iftex}

\newif\ifdevanagariglyphs
\ifXeTeX
  \usepackage{fontspec}
  \IfFontExistsTF{Noto Sans Devanagari}
    {\newfontfamily\devfont{Noto Sans Devanagari}\devanagariglyphstrue}
    {\IfFontExistsTF{Devanagari MT}
      {\newfontfamily\devfont{Devanagari MT}\devanagariglyphstrue}
      {\IfFontExistsTF{Kohinoor Devanagari}
        {\newfontfamily\devfont{Kohinoor Devanagari}\devanagariglyphstrue}
        {\devanagariglyphsfalse}}}
\else
  \devanagariglyphsfalse
\fi

\ifdevanagariglyphs
  \newcommand{\dev}[2]{{\devfont #1}}
  \newcommand{\devrom}[2]{{\devfont #1} (\textit{#2})}
\else
  \newcommand{\dev}[2]{\textit{#2}}
  \newcommand{\devrom}[2]{\textit{#2}}
\fi

\newcommand{\dvdigits}{\dev{०-९}{U+0966--U+096F}}

\newcommand{\ArkiosVocab}{65,536}
\newcommand{\ArkiosNeFert}{1.69}

\newcommand{\QwenNeFert}{6.57}
\newcommand{\QwenVocab}{151,665}

\newcommand{\BloomNeFert}{1.72}
\newcommand{\BloomVocab}{250,680}
\newcommand{\IndicBertNeFert}{1.58}
\newcommand{\IndicBertVocab}{250,000}

\newcommand{\NeBoundControl}{4.67}
\newcommand{\NeBoundTreatment}{1.14}
\newcommand{\NeShatter}{4.09}
\newcommand{\EnShatter}{1.00}
\newcommand{\NumAbugida}{17}

\newcommand{\NumLanguages}{26}
\newcommand{\MaxShatterLang}{Thai}
\newcommand{\MaxShatter}{9.02}
\newcommand{\MinAbugidaShatter}{1.47}
\newcommand{\MinAbugidaShatterLang}{Tibetan}

\newcommand{\PairControlNe}{4.78}
\newcommand{\PairTreatmentNe}{1.58}
\newcommand{\PairSpeedup}{3.03}
\newcommand{\PairControlEn}{1.276}
\newcommand{\PairTreatmentEn}{1.303}
\newcommand{\PairEnCostPct}{2.1}
\newcommand{\PairControlVsBound}{2.2}
\newcommand{\PairVocab}{65,536}
\newcommand{\PairCorpusMB}{2000}
\newcommand{\PairNepaliFrac}{50\%}

\newcommand{\SweepControlSpreadPct}{1.7}

\newcommand{\SweepTreatmentSpreadPct}{33.9}

\newcommand{\SweepControlMaxAboveBoundPct}{3.6}
\newcommand{\SweepEnGapLowPct}{0.25}
\newcommand{\SweepEnGapHighPct}{11.0}
\newcommand{\SweepPoints}{7}
\newcommand{\SweepLow}{5\%}
\newcommand{\SweepHigh}{95\%}
\newcommand{\VocalArabic}{6.50}
\newcommand{\VocalHebrew}{6.14}
\newcommand{\LettersOnlyBestNe}{3.76}

\newcommand{\MarkAwareWorstNe}{4.29}

\newcommand{\FloresSha}{b8b0b7678302}
\newcommand{\HarnessCommit}{156f4527}

\newcommand{\AblANeBpb}{0.4080}

\newcommand{\AblBNeBpb}{0.3899}

\newcommand{\AblCNeBpb}{0.3925}

\newcommand{\AblBvsANe}{4.43}

\newcommand{\AblBvsAEn}{1.84}

\newcommand{\AblBvsCNe}{0.65}

\newcommand{\AblBvsCEn}{1.31}

\newcommand{\AblComputeRatio}{1.59}
\newcommand{\AblComputeFrac}{63}
\newcommand{\AblVerdict}{beyond compression alone}

\newcommand{\MpN}{5}
\newcommand{\MpVocab}{65,536}
\newcommand{\MpCorpusMB}{2,000}
\newcommand{\MpLangs}{Hindi, Bengali, Nepali, Tamil and Malayalam}
\newcommand{\MpMaxBoundGap}{2.2}
\newcommand{\MpMaxBoundGapLang}{Nepali}
\newcommand{\MpMeanBoundGap}{1.8}
\newcommand{\MpMinRatio}{2.54}
\newcommand{\MpMinRatioLang}{Hindi}
\newcommand{\MpMaxRatio}{3.86}
\newcommand{\MpMaxRatioLang}{Tamil}
\newcommand{\MpBoundSpread}{2.2}
\newcommand{\MpTreatAboveBoundMin}{18}
\newcommand{\MpTreatAboveBoundMax}{67}
\newcommand{\MpEnCostMin}{1.6}
\newcommand{\MpEnCostMax}{2.4}

\newcommand{\CensusLimit}{2,000}
\newcommand{\CensusRows}{3,479}
\newcommand{\CensusTgN}{1,345}
\newcommand{\CensusTgTok}{412}
\newcommand{\CensusTgLettersModels}{63.3}
\newcommand{\CensusTgLettersTok}{52.9}
\newcommand{\CensusTgLettersDl}{72.5}

\newcommand{\CensusExcludedNoTok}{797}

\newcommand{\code}[1]{\texttt{#1}}
\newcommand{\pL}{\code{\textbackslash p\{L\}+}}
\newcommand{\pLM}{\code{[\textbackslash p\{L\}\textbackslash p\{M\}]+}}
\newcommand{\cat}[1]{\textsf{#1}}

\makeatletter
\newif\ifdeanon
\ifacl@anonymize\deanonfalse\else\deanontrue\fi
\makeatother

\ifdeanon\newcommand{\sysname}{Arkios}\else\newcommand{\sysname}{Ours}\fi

\ifdeanon
  \newcommand{\repourl}{https://github.com/sajalregmi/arkios-tokenizer}
  \newcommand{\hfurl}{https://huggingface.co/sajalregmi4/arkios-tokenizer}
  \newcommand{\anonurl}[1]{\url{#1}}
\else
  \newcommand{\repourl}{}
  \newcommand{\hfurl}{}
  \newcommand{\anonurl}[1]{\textit{(URL withheld for anonymous review)}}
\fi

\newtheorem{proposition}{Proposition}

\title{Vowel Signs Are Not Letters:\\
       A Pre-tokenization Ceiling on Multilingual Tokenizer Fertility}

\author{
  Sajal Regmi\thanks{Corresponding author and primary contributor.
  \texttt{sajalregmi@karelatechnologies.com}} \quad
  Siddhartha Pudasaini \quad
  Chetan Phakami Pun \\[4pt]
  \normalsize Karela Technologies Inc.
}
\date{}

\begin{document}
\maketitle

\begin{abstract}
Byte-level BPE tokenizers that use the HuggingFace \code{ByteLevel}
pre-tokenizer inherit GPT-2's word regex, where a word is defined as \pL{}, one
or more Unicode \emph{letters}. In abugida scripts, vowels are written as
combining \emph{marks}; this pattern therefore splits each word at every vowel
sign. Since BPE merges only within a pre-token, those splits persist through
training regardless of vocabulary size or corpus composition. We formalise this
effect as a training-free lower bound on fertility. Across \NumLanguages{}
languages from a parallel corpus, every one of the \NumAbugida{} abugidas is
affected, ranging from $\MinAbugidaShatter\times$ (\MinAbugidaShatterLang) to
$\MaxShatter\times$ (\MaxShatterLang), whereas Latin, Cyrillic, Hangul, and Han
show exactly $1.00\times$. For \MpN{} languages, matched tokenizer pairs that
differ only in this character class fall within \MpMaxBoundGap\% of the predicted
floor, scoring \PairControlNe{} versus \PairTreatmentNe{} tokens per word on
Nepali. When the Nepali share of the training corpus is swept from \SweepLow{}
to \SweepHigh{}, the broken tokenizer barely shifts at all
(\SweepControlSpreadPct\%) while the fixed one shifts \SweepTreatmentSpreadPct\%,
which separates a
structural ceiling from a data shortage without needing to inspect any code. We
train three 268M models that differ only in their tokenizer; the fixed variant
achieves \AblBvsANe\% lower held-out Nepali bits per byte at equal compute, and
it still leads when given the same bytes with $\AblComputeRatio\times$ the
compute. A census of \CensusRows{} HuggingFace repositories finds the
letters-only word class present in \CensusTgLettersModels\% of the
most-downloaded text-generation models, accounting for
\CensusTgLettersDl\% of their downloads. GPT-4o's \code{o200k} pattern already
uses a mark-aware word class, making the repair itself prior art. We quantify
its value, show how to recognise its absence from symptoms alone, map which
scripts it reaches, measure how widely it is deployed, and release a
\ArkiosVocab-entry Nepali--English tokenizer with a harness that regenerates
every number here from public data on a laptop.
\end{abstract}

\section{Introduction}

Fertility, the mean number of tokens a tokenizer spends per word, sets an
exchange rate between a language and every cost tied to a language model. With
context length and training budget held fixed, a language tokenized at 5 tokens
per word gets a fifth as much effective text as one tokenized at 1. For
low-resource languages the corpus is already the binding constraint, and
fertility compounds it.

We built a bilingual English-Nepali tokenizer and found Nepali fertility near
4.4 tokens per word. The standard diagnosis was a data shortage: Nepali is
low-resource, the mixture was English-dominant, and adding more Nepali is the
usual remedy. We raised the Nepali share and retrained. The number did not
move. We raised it again, and it still had no effect.

The cause was one character class. The HuggingFace \code{ByteLevel}
pre-tokenizer applies the regex from GPT-2 \citep{radford2019gpt2}, whose word
alternative is \pL{}, matching one or more Unicode letters. In Devanagari,
vowels are written as combining marks. The vowel sign \textsc{aa} falls under
category \cat{Mc}, \textsc{e} under \cat{Mn}, and the virama under \cat{Mn} as
well. None are \cat{L}, so \pL{} splits a Nepali word at every vowel sign. BPE,
which merges pairs only inside a pre-token, has no way to rejoin them. Under
\pL{}, the word \devrom{नेपाली}{nep\={a}l\={\i}} yields six pre-tokens; under
\pLM{}, just one.

The \code{o200k} pattern shipped with GPT-4o already uses a mark-aware word
class, and SentencePiece \citep{kudo2018sentencepiece} sidesteps letter-based
pre-tokenization entirely. This is not an unknown defect. We report what it
costs, how to recognise it from its symptoms, and how far it reaches across the
world's writing systems.

\paragraph{Contributions.}
\begin{enumerate}
\item A lower bound on fertility depending only on the pre-tokenizer, not on
      the vocabulary, corpus, or merge count (\S\ref{sec:bound}). It takes one
      regex match to evaluate.
\item A diagnostic (\S\ref{sec:sweep}). As the Nepali share of the training
      corpus is swept from \SweepLow{} to \SweepHigh{}, fertility shifts by
      \SweepControlSpreadPct\% under \pL{} and \SweepTreatmentSpreadPct\% under
      \pLM{}. A metric that does not respond to its supposed cause is
      constrained upstream of it, and the same logic extends well beyond
      tokenizers.
\item A controlled measurement, repeated \MpN{} times (\S\ref{sec:pair}). Two
      tokenizers differing in one character class score \PairControlNe{} and
      \PairTreatmentNe{} tokens per word on Nepali, at a cost of
      \PairEnCostPct\% of English fertility, which we trace to vocabulary
      reallocation rather than the regex. Repeated on Hindi, Bengali, Tamil and
      Malayalam, chosen in advance to span $\MpBoundSpread\times$ in predicted
      effect, every control arm falls within \MpMaxBoundGap\% of the bound,
      making Proposition~\ref{prop:bound} a predictor and not just a theorem.
\item A census rather than a shortlist (\S\ref{sec:census}): across
      \CensusRows{} HuggingFace repositories, \CensusTgLettersModels\% of the
      most-downloaded text-generation models carry a letters-only pre-tokenizer,
      accounting for \CensusTgLettersDl\% of their downloads.
\item The defect's scope (\S\ref{sec:scripts}): every abugida we tested, plus
      Arabic and Hebrew once their diacritics are present, where ratios reach
      $\VocalArabic{}\times$ and $\VocalHebrew{}\times$.
\item A \ArkiosVocab-entry Nepali--English tokenizer, with a harness that
      regenerates every number in this paper from public data, using no GPU and
      no account (\S\ref{sec:repro}).
\end{enumerate}

Two things we do not claim. We do not report the best Nepali fertility
available at any cost, since a 250{,}000-entry encoder tokenizer beats ours and
\S\ref{sec:cross} reports it. And we do not claim the repair is free:
\S\ref{sec:pair} measures what it costs English.

\section{Background}

\paragraph{Pre-tokenization.} BPE tokenizers do not apply BPE directly to raw
text: a regex first segments the input into chunks, and merges are learned and
applied independently within each chunk. GPT-2 introduced the pattern that most
byte-level pipelines still employ, whose word alternative matches \pL{} with an
optional leading space. Pre-tokenization exists to prevent merges from spanning
word boundaries, ensuring that no single token spells \code{``the cat''}; it
enforces this by rendering the boundaries impassable. Byte-level BPE
\citep{sennrich2016bpe, radford2019gpt2} maps text to UTF-8 bytes, ensuring no
input is out-of-vocabulary, and repeatedly merges the most frequent adjacent
pair. Byte-level coverage ensures that every script can be \emph{represented},
but ensures nothing about how efficiently.

\paragraph{Abugidas.} An alphabet represents vowels as letters. An abugida
assigns each consonant an inherent vowel and represents any other vowel as a
mark attached to that consonant. Unicode assigns those marks to categories
\cat{Mn} (non-spacing) and \cat{Mc} (spacing combining), separate from \cat{L}
\citep{unicode}. This separation is correct, since a vowel sign is not an
independent letter, and it is precisely what \pL{} fails to account for.

\paragraph{Fertility.} We report tokens per whitespace-delimited word alongside
bytes per token. Whitespace does not delimit words in Thai, Khmer, Lao, Myanmar,
Tibetan, Chinese, or Japanese; for those scripts we report tokens per 100
characters and leave fertility undefined, rather than reporting a number that
invites a comparison it cannot support.

\section{A bound that requires no training}
\label{sec:bound}

The mechanism is visible at the code-point level: three of the six characters in
\devrom{नेपाली}{nep\={a}l\={\i}} are marks, so \pL{} produces six single-character
pre-tokens where \pLM{} produces one (Table~\ref{tab:codepoints},
Appendix~\ref{app:codepoints}).

\begin{proposition}
\label{prop:bound}
Let $P$ be a pre-tokenizer and let $T$ be any BPE tokenizer whose merges are
learned and applied within the pre-tokens $P$ produces. Then for every string
$s$, $|T(s)| \ge |P(s)|$.
\end{proposition}

\begin{proof}
Every merge replaces an adjacent pair within a single pre-token with one token,
and no merge spans two pre-tokens. The number of pieces is thus at least the
number of pre-tokens, for any vocabulary size, corpus, mixture, or number of
merges.
\end{proof}

Proposition~\ref{prop:bound} transforms a regex into a floor. Under \pL{},
Nepali is split into \NeShatter{} times as many pre-tokens as under \pLM{},
placing the fertility floor at \NeBoundControl{} tokens per word against
\NeBoundTreatment{}, before any training. No quantity of Nepali data can reach
below that floor, because the floor is not made of data.

\section{Diagnosis: a sweep that does not respond}
\label{sec:sweep}

Poor fertility on a low-resource language invites a mixture hypothesis, and
this hypothesis has an obvious test: sweep the language's share of the tokenizer
training corpus and observe the metric. We conducted that sweep under both word
classes at the full configuration, \PairVocab{} vocabulary and
\PairCorpusMB\,MB corpus, varying only the Nepali byte share.

\begin{figure}[t]
\centering
\begin{tikzpicture}
\begin{axis}[width=\linewidth, height=6.4cm,
  xlabel={Nepali share of the tokenizer training corpus},
  ylabel={Nepali fertility (tokens/word)},
  xtick={0,0.2,0.4,0.6,0.8,1.0},
  xticklabels={0\%,20\%,40\%,60\%,80\%,100\%},
  legend pos=north east, grid=major, ymin=0, ymax=7.02,
  legend cell align=left]
\addplot[mark=*, thick] coordinates {(0.05,4.8406) (0.20,4.7938) (0.35,4.7819) (0.50,4.7753) (0.65,4.7713) (0.80,4.7639) (0.95,4.7599)};
\addlegendentry{$\backslash$p\{L\}+ (control)}
\addplot[mark=square*, thick, dashed] coordinates {(0.05,1.9660) (0.20,1.7041) (0.35,1.6253) (0.50,1.5786) (0.65,1.5405) (0.80,1.5054) (0.95,1.4684)};
\addlegendentry{[$\backslash$p\{L\}$\backslash$p\{M\}]+ (treatment)}
\addplot[dotted, thick, domain=0:1, samples=2] {4.6726};
\addlegendentry{pre-tokenization bound (control)}
\end{axis}
\end{tikzpicture}
\caption{Nepali fertility against the Nepali share of the tokenizer training
corpus, \SweepPoints{} points per curve, each representing a tokenizer trained
from scratch. The control remains on its pre-tokenization bound across the
entire range. The treatment declines as Nepali content rises. The two curves
differ by one character class.}
\label{fig:sweep}
\end{figure}
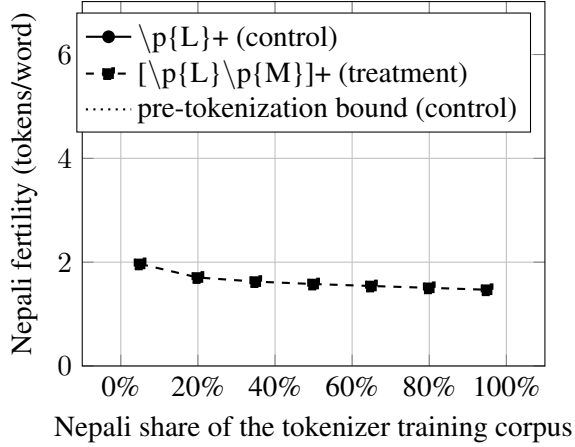

The two curves differ in shape, not merely in level. Under \pL{}, fertility
varies by \SweepControlSpreadPct\% across a \SweepLow--\SweepHigh{} range of
Nepali content and remains within \SweepControlMaxAboveBoundPct\% of the
training-free bound throughout. Under \pLM{} it varies by
\SweepTreatmentSpreadPct\% and declines monotonically as Nepali content rises.
Table~\ref{tab:sweep} in Appendix~\ref{app:sweep} reports the numbers behind both
curves.

A flat response is evidence against the mixture hypothesis, not weak evidence
for it. A data-limited quantity responds to data; one that does not respond is
constrained upstream of the corpus, and a tokenizer pipeline has few components
upstream of the corpus: the normaliser, the pre-tokenizer, and the vocabulary
budget. Examining those three took an afternoon. Acquiring more Nepali text,
which is what we tried first, took considerably longer and could not have
succeeded.

\section{Matched pairs}
\label{sec:pair}

To isolate the effect of the word class, we trained two tokenizers differing in
a single character class. Both consume the same corpus bytes, verified by sha256,
under the same NFC normaliser, digit pre-split, byte-level mapping, trainer
settings, special tokens, and \PairVocab{} vocabulary. Both are produced by the
same training function invoked with a different \code{word\_class} argument, so
the control amounts to the production recipe with the repair removed rather than
a reimplementation of it.

Nepali fertility drops from \PairControlNe{} to \PairTreatmentNe{} tokens per
word, a factor of \PairSpeedup. The control settles within
\PairControlVsBound\% of its training-free bound of \NeBoundControl{}, placing
the corpus well clear of the binding constraint.

\paragraph{English pays \PairEnCostPct\%.} English fertility shifts in the
opposite direction, from \PairControlEn{} to \PairTreatmentEn{} tokens per word.
English \emph{pre-tokenization} is unchanged: the shatter ratio for Latin script
is exactly $\EnShatter\times$, so both arms segment English text into identical
pre-tokens. The regression stems from the vocabulary budget. Once Nepali words
survive pre-tokenization intact, they begin winning merge slots, and at fixed
$|V|$ those slots are taken from English. The sweep separates the two effects:
at \SweepLow{} Nepali the arms differ on English by \SweepEnGapLowPct\%, and at
\SweepHigh{} by \SweepEnGapHighPct\%. The repair is free at the regex level and
costs approximately \PairEnCostPct\% of English fertility at a fixed vocabulary.

\paragraph{Four more languages, chosen by prediction.} A single pair is a single
data point, and Proposition~\ref{prop:bound} makes a claim about every language.
We repeated the identical experiment for Hindi, Bengali, Tamil, and Malayalam:
the same \MpCorpusMB\,MB corpus, the same 50\% target-language byte share, the
same \MpVocab{} vocabulary, the same training function, with one argument
changed. The four were selected before any of them was trained, and selection
was based on the bound rather than on the outcome: they span
$\MpBoundSpread\times$ in \emph{predicted} control fertility, Hindi lowest and
Malayalam highest. Hindi earns its inclusion precisely because it is where the
predicted effect is smallest. A rule that holds only where the effect is large
is of limited value.

\begin{table*}[t]
\centering
\small
\setlength{\tabcolsep}{4pt}
\begin{tabular}{llrrrrrrr}
\toprule
& & \multicolumn{2}{c}{Bound (tok/word)} & \multicolumn{2}{c}{Measured (tok/word)} & & & \\
\cmidrule(lr){3-4}\cmidrule(lr){5-6}
Language & Script & \pL{} & \pLM{} & \pL{} & \pLM{} & $\Delta$ bound & Ratio & English \\
\midrule
Hindi & Devanagari & 3.30 & 1.12 & 3.36 & 1.32 & +1.8\% & 2.54$\times$ & +1.6\% \\
Bengali & Bengali & 4.65 & 1.15 & 4.73 & 1.50 & +1.8\% & 3.16$\times$ & +2.2\% \\
Nepali & Devanagari & 4.67 & 1.14 & 4.78 & 1.58 & +2.2\% & 3.03$\times$ & +2.1\% \\
Tamil & Tamil & 6.76 & 1.19 & 6.85 & 1.77 & +1.4\% & 3.86$\times$ & +2.2\% \\
Malayalam & Malayalam & 7.34 & 1.26 & 7.48 & 2.11 & +1.9\% & 3.54$\times$ & +2.4\% \\
\bottomrule
\end{tabular}

\caption{Five matched pairs on FLORES-200 devtest, one row per target language,
each \MpCorpusMB\,MB at 50\% target language by bytes with a \MpVocab{}
vocabulary. ``Bound'' denotes the training-free floor of \S\ref{sec:bound};
``$\Delta$ bound'' denotes how far the trained control arm settles above its own
floor; ``Ratio'' is measured control fertility over measured treatment
fertility; ``English'' is the change in English fertility from control to
treatment. Nepali is the headline pair of \S\ref{sec:pair};
Table~\ref{tab:pair} in Appendix~\ref{app:pair} provides its bytes-per-token
figures.}
\label{tab:multipair}
\end{table*}

The column of interest is $\Delta$ bound. The control arm settles
\MpMeanBoundGap\% above its own training-free floor on average and never more
than \MpMaxBoundGap\% above it (\MpMaxBoundGapLang).
Proposition~\ref{prop:bound} guarantees the sign of that gap but says nothing
about its magnitude: a floor that trained tokenizers missed by 40\% would be a
valid theorem and a useless predictor, and measuring the gap across a
$\MpBoundSpread\times$ range of predicted values is what distinguishes the two.
The explanation is mechanical rather than statistical: under \pL{} the
pre-tokenizer has already made every available cut, so BPE has almost nothing
left to merge and the floor is the answer.

The treatment arm behaves differently, as it should: it settles
\MpTreatAboveBoundMin--\MpTreatAboveBoundMax\% above \emph{its} own bound,
because once whole words survive pre-tokenization the binding constraint shifts
to the vocabulary budget, where a designer can act on it. The measured
control-to-treatment ratio spans from $\MpMinRatio\times$ (\MpMinRatioLang) to
$\MpMaxRatio\times$ (\MpMaxRatioLang), and English pays between \MpEnCostMin\%
and \MpEnCostMax\% across the five. One recipe detail: the digit pre-split,
\code{[0-9\dvdigits]\{1,3\}}, covers ASCII and Devanagari digits but not Tamil,
Bengali, or Malayalam ones. It is identical across both arms, so it cannot
influence any control--treatment contrast, but it does mean these tokenizers are
our production recipe transplanted rather than tuned per language; the experiment
measures one character class, not five deployable tokenizers.

\section{Comparison with released tokenizers}
\label{sec:cross}

Table~\ref{tab:main} compares our tokenizer against thirteen released tokenizers
on FLORES-200 devtest \citep{nllb2022flores}, using byte-identical NFC input for
every system. Because FLORES-200 is parallel across 204 languages, every
language examined in this paper is measured on the same content; that is what
gives the multi-script comparison in \S\ref{sec:scripts} its meaning.

\begin{table*}[t]
\centering
\footnotesize
\setlength{\tabcolsep}{4pt}
\begin{tabular}{llrrrrrc}
\toprule
& & & \multicolumn{2}{c}{Nepali} & \multicolumn{2}{c}{English} & \\
\cmidrule(lr){4-5}\cmidrule(lr){6-7}
Tokenizer & Family & $|V|$ & tok/word & B/tok & tok/word & B/tok & Lossless \\
\midrule
Llama-3 & English-centric & 128,256 & 3.76 & 4.84 & 1.24 & 4.90 & yes \\
DeepSeek-V3 & English-centric & 128,815 & 4.29 & 4.24 & 1.23 & 4.92 & yes \\
Qwen2.5 & English-centric & 151,665 & 6.57 & 2.77 & 1.26 & 4.82 & yes \\
cl100k (GPT-3.5/4) & English-centric & 100,277 & 6.98 & 2.61 & 1.24 & 4.90 & yes \\
GPT-2 & English-centric & 50,257 & 10.97 & 1.66 & 1.28 & 4.74 & yes \\
\addlinespace
o200k (GPT-4o) & Frontier & 200,019 & 2.32 & 7.84 & 1.23 & 4.95 & yes \\
Gemma-2 & Frontier & 256,000 & 3.13 & 5.81 & 1.28 & 4.75 & yes \\
Mistral NeMo (Tekken) & Frontier & 131,072 & 3.17 & 5.75 & 1.27 & 4.78 & yes \\
\addlinespace
IndicBERTv2 & Massively multi. & 250,000 & \textbf{1.58} & 11.54 & 1.24 & 4.92 & yes \\
BLOOM & Massively multi. & 250,680 & 1.72 & 10.56 & 1.25 & 4.85 & yes \\
NLLB-200 & Massively multi. & 256,204 & 1.92 & 9.48 & 1.40 & 4.34 & \textbf{no} \\
mT5 & Massively multi. & 250,100 & 2.64 & 6.90 & 1.54 & 3.94 & \textbf{no} \\
\addlinespace
Sarvam-1 & Indic & 68,096 & 2.66 & 6.85 & 1.50 & 4.06 & yes \\
\addlinespace
\sysname{} & Ours & 65,536 & 1.69 & 10.79 & 1.30 & 4.69 & yes \\
\bottomrule
\end{tabular}

\caption{Fertility and bytes per token on FLORES-200 devtest. ``Lossless''
indicates whether decode(encode($x$)) preserves every character. Bold marks the
best lossless Nepali fertility.}
\label{tab:main}
\end{table*}

\paragraph{Fidelity is part of the measurement.} A tokenizer that drops
characters it cannot represent earns a favourable fertility on text it cannot
encode. We therefore measure, for each tokenizer and language, the fraction of
input characters absent after a decode round-trip. NLLB-200 deletes curly
quotation marks, en-dashes, and ZWJ/ZWNJ, discarding 0.20\% of Nepali characters;
mT5 deletes ZWJ and ZWNJ. Both are flagged in Table~\ref{tab:main}, and their
fertilities are not comparable with the rest. The same check guards against the
opposite error: IndicBERTv2's decoder emits WordPiece \code{\#\#}
continuation markers, which a naive round-trip test would score as corruption.
It is lossless on every language we examine, and it outperforms us.

\paragraph{We do not win outright.} IndicBERTv2 reaches \IndicBertNeFert{}
tokens per word against our \ArkiosNeFert{}, with a \IndicBertVocab-entry
WordPiece vocabulary roughly four times ours; it is an encoder tokenizer that
does not preserve whitespace, which rules it out as a generative drop-in but not
as a fertility baseline. BLOOM, with \BloomVocab{} entries, reaches
\BloomNeFert{}. Our defensible claim is narrower than beating the frontier:
among tokenizers suitable for generative modelling, ours attains the lowest
Nepali fertility in the panel at a quarter of the next best's vocabulary.

\paragraph{The word class sorts the table.} Every letters-only tokenizer scores
\LettersOnlyBestNe{} or worse on Nepali; everything below that threshold is
mark-aware or does not split on letters at all, and the ordering is not a
vocabulary-size effect (Qwen2.5 carries \QwenVocab{} entries and still requires
\QwenNeFert{}). The word class is necessary but not sufficient, which the
matched pair of \S\ref{sec:pair} is designed to isolate: DeepSeek-V3 is
mark-aware and still requires \MarkAwareWorstNe{}, because its training text
contains little Devanagari.

\section{How much of the ecosystem inherits the pattern}
\label{sec:census}

Table~\ref{tab:main} covers thirteen tokenizers that we selected, which is
precisely the kind of evidence a reader should discount. To replace selection
with a rule, we classified the pre-tokenizer of every model in two listings from
the HuggingFace Hub \citep{wolf2020transformers}: the top \CensusLimit{}
repositories by 30-day downloads, and the top \CensusLimit{} carrying the
\code{text-generation} pipeline tag, \CensusRows{} distinct repositories in
total. Classification uses the same function as the rest of this paper, applied
to the \code{pre\_tokenizer} object extracted from the first 512\,kB of each
\code{tokenizer.json} via an HTTP range request. The vocabulary and merges, which
constitute the bulk of the file, are never fetched, so the census consumes a few
hundred megabytes rather than several terabytes. Repositories lacking a
\code{tokenizer.json} are classified from their file layout instead: a
SentencePiece model file implies no letter-based pre-tokenization and counts as
unaffected, while \code{vocab.json} with \code{merges.txt} is the
pre-\code{tokenizers} GPT-2 serialisation, which \code{transformers} loads
through \code{ByteLevel} with the regex enabled, and counts as affected. Dropping
either group would bias the result, and the two biases point in opposite
directions.

\begin{table*}[t]
\centering
\small
\setlength{\tabcolsep}{4pt}
\begin{tabular}{lrrrr}
\toprule
& & \multicolumn{3}{c}{Letters-only, as \% of\ldots} \\
\cmidrule(lr){3-5}
Population & Classified & repos & distinct tok. & downloads \\
\midrule
Text generation, top-$N$ & 1345 & 63.3\% & 52.9\% & 72.5\% \\
\quad of those, affected language tag & 78 & 67.9\% & 56.8\% & 60.8\% \\
All repositories, top-$N$ & 1126 & 45.2\% & 42.2\% & 30.3\% \\
\quad of those, affected language tag & 149 & 50.3\% & 40.5\% & 18.9\% \\
\bottomrule
\end{tabular}

\caption{Share of repositories using a letters-only pre-tokenizer, under three
denominators: repositories, distinct \code{tokenizer.json} files (deduplicated
by content hash, so a base model with many fine-tunes counts once; repositories
that ship no such file have nothing to deduplicate and are omitted from that
column alone), and 30-day downloads. Percentages are of the \emph{classified}
repositories in each row; repositories with no tokenizer at all,
\CensusExcludedNoTok{} of the overall top-\CensusLimit{}, comprising vision,
audio, and diffusion models and weights-only redistributions, as well as gated
ones, are excluded from every denominator.}
\label{tab:census}
\end{table*}

Among text-generation models, \CensusTgLettersModels\% of the \CensusTgN{}
classified repositories carry a letters-only pre-tokenizer. Those repositories
account for \CensusTgLettersTok\% of the \CensusTgTok{} distinct tokenizers
behind them and \CensusTgLettersDl\% of 30-day downloads. The three denominators
are reported jointly because each is vulnerable to a distinct objection:
repository counts inflate whichever base model has the most fine-tunes,
distinct-tokenizer counts give a niche tokenizer the same weight as one used by
half the ecosystem, and download counts track whatever a few large deployments
happen to pull this month. They agree here, which is the useful outcome.

\paragraph{A field whose default is easy to get wrong.} HuggingFace's
\code{ByteLevel} pre-tokenizer applies GPT-2's pattern internally unless
\code{use\_regex} is set to false, and the field is typically \emph{absent} from
the file. GPT-2's own \code{tokenizer.json} omits it, and the library fills in
true, after which \devrom{नेपाली}{nep\={a}l\={\i}} emerges as six pieces.
Reading an absent field as ``disabled'' would classify the canonical instance of
this defect as unaffected, and would do the same for every file written before
the field existed. Our classifier treats absence as enabled, matching the
library's behaviour; the appendix enumerates the distinct patterns behind these
counts.

We make no larger claim than the one these counts support: the letters-only word
class is what the ecosystem ships by default, not a historical curiosity that
has been cleaned up. The Limitations section sets out what the census does not establish.

\section{Which scripts the defect reaches}
\label{sec:scripts}

Because the bound requires no training, we assess its reach across every
language in a parallel corpus at negligible cost. Table~\ref{tab:bounds} reports
the shatter ratio, pre-tokens under \pL{} divided by pre-tokens under \pLM{},
for all \NumLanguages{} languages on identical content.

\begin{table}[t]
\centering
\footnotesize
\setlength{\tabcolsep}{3.5pt}
\begin{tabular}{lrrr}
\toprule
& & \multicolumn{2}{c}{Bound, tok/word} \\
\cmidrule(lr){3-4}
Language & Shatter & \pL{} & \pLM{} \\
\midrule
Thai & \textbf{9.02$\times$} & -- & -- \\
Burmese & \textbf{7.58$\times$} & -- & -- \\
Lao & \textbf{5.95$\times$} & -- & -- \\
Malayalam & \textbf{5.82$\times$} & 7.34 & 1.26 \\
Tamil & \textbf{5.66$\times$} & 6.76 & 1.19 \\
Khmer & \textbf{5.11$\times$} & -- & -- \\
Kannada & \textbf{4.86$\times$} & 6.09 & 1.25 \\
Telugu & \textbf{4.57$\times$} & 5.65 & 1.24 \\
Marathi & \textbf{4.25$\times$} & 4.93 & 1.16 \\
Odia & \textbf{4.14$\times$} & 4.81 & 1.16 \\
Nepali & \textbf{4.09$\times$} & 4.67 & 1.14 \\
Bengali & \textbf{4.03$\times$} & 4.65 & 1.15 \\
Gujarati & \textbf{3.51$\times$} & 4.03 & 1.15 \\
Sinhala & \textbf{3.25$\times$} & 4.12 & 1.27 \\
Hindi & \textbf{2.95$\times$} & 3.30 & 1.12 \\
Punjabi & \textbf{2.87$\times$} & 3.25 & 1.13 \\
Tibetan & \textbf{1.47$\times$} & -- & -- \\
\addlinespace
Arabic & \textbf{1.08$\times$} & 1.22 & 1.13 \\
\addlinespace
\emph{5 alphabets} & 1.00$\times$ & \multicolumn{2}{c}{\emph{identical}} \\
\addlinespace
\emph{Hebrew} & 1.00$\times$ & \multicolumn{2}{c}{\emph{identical}} \\
\addlinespace
\emph{2 logographics} & 1.00$\times$ & \multicolumn{2}{c}{\emph{identical}} \\
\bottomrule
\end{tabular}

\caption{The training-free pre-tokenization bound across FLORES-200 devtest.
Shatter ratio is pre-tokens under \pL{} over pre-tokens under \pLM{};
$1.00\times$ means the two word classes are indistinguishable, and those
languages are grouped by family. Dashes mark the languages that do not separate
words with spaces, where tok/word is undefined;
Table~\ref{tab:boundsfull} in Appendix~\ref{app:bounds} reports those in tokens
per 100 characters and provides the full per-language grid. The corpus is
parallel, so every row describes the same content.}
\label{tab:bounds}
\end{table}

All \NumAbugida{} abugidas are affected, from $\MinAbugidaShatter\times$
(\MinAbugidaShatterLang) to $\MaxShatter\times$ (\MaxShatterLang); Latin,
Cyrillic, Hangul, and Han sit at exactly $1.00\times$, since ordinary text in
those scripts carries no combining marks after NFC.

\paragraph{Tibetan is a partial exception.} At the low end, Tibetan separates
syllables with U+0F0B \textsc{tsheg} (punctuation, \cat{Po}), which splits under
both word classes, so its pre-tokens are syllables either way and the marks
merely fragment an already short unit. The defect tracks how much of a word's
information Unicode classifies as marks, not whether a script is Brahmic.

\paragraph{Arabic and Hebrew are dormant cases.} Table~\ref{tab:bounds} reports
Arabic at $1.08\times$ and Hebrew at exactly $1.00\times$, because the evaluation
text omits their diacritics, as modern prose does. Vocalised Arabic and pointed
Hebrew appear in scripture, poetry, dictionaries, and language-learning
material, much of the text a low-resource pipeline is handed. Measured on the
same sentence with its marks restored (Table~\ref{tab:vocal},
Appendix~\ref{app:codepoints}), vocalised Arabic shatters $\VocalArabic{}\times$
and pointed Hebrew $\VocalHebrew{}\times$, comparable to Devanagari. A pipeline
audited on modern newswire will pass and then fail on the first vocalised corpus
it encounters.

\section{Downstream evaluation}
\label{sec:downstream}

Fertility measures compression. Whether the resulting model is actually better
is a separate question.

\paragraph{Design.} Two tokenizers split the same text into different numbers of
tokens, so per-token loss measures different quantities and cross-tokenizer
perplexity is meaningless; every figure here is \textbf{bits per byte}, which
normalises by the underlying text. Three runs share an architecture (268M
parameters, 16 layers, $d_{\text{model}}=1024$, context 2048), an optimiser
configuration, and a seed. Run A trains the broken tokenizer on $N$ tokens; B
trains the fixed tokenizer on the same $N$; C trains the broken tokenizer on
enough tokens to read the same \emph{bytes} as B. At a fixed token budget, B
reads approximately $\PairSpeedup\times$ more text than A, so an advantage for B
could be attributed to having read more; C eliminates that explanation at a
cost of $\AblComputeRatio\times$ B's compute. A and C share one shard set, C's
shards being A's read further, so they see identical data in identical order.
Evaluation relies on held-out splits disjoint from training: FineWeb-2's
designated \code{npi\_Deva} test split for Nepali, and C4's validation split for
English, a different corpus from the FineWeb-Edu the models train on.

\begin{table}[t]
\centering
\scriptsize
\setlength{\tabcolsep}{3pt}
\begin{tabular}{llrrr}
\toprule
& & \multicolumn{2}{c}{Held-out bits/byte} & \\
\cmidrule(lr){3-4}
Condition & Budget & Nepali & English & Train \\
\midrule
A~~broken \pL & 4.00B & 0.4080 & 1.1706 & 0.7125 \\
B~~fixed \pLM & 4.00B & \textbf{0.3899} & 1.1490 & 0.6720 \\
C~~broken \pL & 6.36B & 0.3925 & 1.1342 & 0.6060 \\
\bottomrule
\end{tabular}

\caption{Held-out bits per byte; lower is better. C receives
$\AblComputeRatio\times$ B's compute by construction.}
\label{tab:ablation}
\end{table}

\paragraph{Fixing the regex improves the model.} At equal compute, B reaches
\AblBNeBpb{} bits per byte on held-out Nepali against A's \AblANeBpb{}, an
improvement of \AblBvsANe\%. English improves as well, by \AblBvsAEn\%, even
though \S\ref{sec:pair} shows the fixed tokenizer costs \PairEnCostPct\% of
English \emph{fertility}. Both arms spend the same token budget, and B's larger
tokens mean it reads more text in both languages; the fertility cost is real,
yet the downstream effect still runs the other way.

\paragraph{English is the control channel, and it behaves.} The regex provably
does nothing to Latin script: the shatter ratio is exactly $\EnShatter\times$
(\S\ref{sec:scripts}), so B and C segment English into identical pre-tokens and
differ only in compute. On English they order by compute, C ahead of B by
\AblBvsCEn\%. That is the ordering a control should produce, and it is what
licenses reading the Nepali comparison as something other than noise.

\paragraph{On Nepali the ordering inverts.} B reaches \AblBNeBpb{} against C's
\AblCNeBpb{} while consuming \AblComputeFrac\% of C's compute. We state the
result in that form deliberately. The raw gap is \AblBvsCNe\% on a single seed
with no error bars, which is too thin to defend on its own; the compute-normalised
statement does not depend on the gap's magnitude, only on its sign, and the
English channel shows what the sign would have been had compute alone been
responsible. Read together, the two channels place the benefit \AblVerdict{}.

\paragraph{What this does not establish.} One seed per condition, one model
scale, one language pair, and bits per byte rather than task accuracy: a reader
who wants a claim about downstream task quality will not find it here. What the
three conditions support is narrower, namely that at a fixed compute budget the
fixed tokenizer produces a better Nepali model, and that advantage survives
handing the broken tokenizer the same bytes and half again as much compute.

\section{Reproducibility}
\label{sec:repro}

A single command regenerates every number in this paper, including those quoted
in prose:

\begin{quote}
\code{make paper-harness}
\end{quote}

\ifdeanon
The harness, the manuscript source, and the released tokenizer are available at
\anonurl{\repourl} and \anonurl{\hfurl}.
\else
The harness, the manuscript source, and the released tokenizer will be released
publicly; the URL is withheld here to preserve anonymity.
\fi

It requires no GPU, no HuggingFace account, and no private data. It downloads
FLORES-200 as a plain tarball (sha256 \code{\FloresSha}\ldots), streams a public
Nepali--English corpus, trains the matched arms and the sweep, evaluates each
tokenizer, and writes a machine-readable \code{results.json} together with the
\LaTeX{} tables this document includes. A \code{-{}-quick} configuration
completes in about fifteen minutes and reproduces the bound across
\NumLanguages{} languages and the vocalisation probe exactly, since neither
depends on the trained arms.

\code{results.json} records the resolved Hub commit of every baseline, the
sha256 of every corpus, the version of every library, and the harness's git
commit (\code{\HarnessCommit}); Appendix~\ref{app:repro} provides the corpora,
splits, and compute budget in full. Every tokenizer result here ran on a laptop
CPU. Only the downstream evaluation (\S\ref{sec:downstream}) requires a GPU:
three runs totalling 14.36B tokens over 268M-parameter models
($2.3\times10^{19}$ FLOPs by the $6ND$ estimate of \citealp{kaplan2020scaling}),
\textbf{19.1 GPU-hours on a single NVIDIA H100 80GB SXM5}. The census
(\S\ref{sec:census}) is a separate command, because it reads a live service
rather than a pinned artifact: rankings and download counts shift, so a rerun
will not reproduce our counts exactly, though every repository we classified is
recorded in \code{census.json}.

\section{Related work}

Byte-pair encoding originated as a compression algorithm \citep{gage1994bpe},
entered NMT through \citet{sennrich2016bpe}, and became byte-level in
\citet{radford2019gpt2}, whose pre-tokenization regex is the subject of this
paper (\citealp{mielke2021between} survey the design space). SentencePiece
\citep{kudo2018sentencepiece} avoids letter-based pre-tokenization and the
failure that follows from it, at the cost of the fidelity issues our
measurements reveal; tokenizer-free models sidestep the question entirely by
operating on bytes or characters \citep{xue2022byt5, clark2022canine}, at a
compute cost byte-level BPE exists to avoid. Indic tokenization has received
targeted attention, including the IndicNLP suite
\citep{kakwani2020indicnlpsuite}, and massively multilingual models have long
contended with script diversity \citep{conneau2020xlmr}; Devanagari
fragmentation under English-centric tokenizers has been observed before.

Tokenizer quality outside English is well documented.
\citet{rust2021good} show dedicated monolingual tokenizers outperform
multilingual ones; \citet{ahia2023all} and \citet{petrov2023language} quantify
the resulting cost and equity gap across languages;
\citet{limisiewicz2023tokenization} study how vocabulary allocation propagates
into model quality; \citet{ali2024tokenizer} show tokenizer choice materially
affects downstream performance and cost. \citet{zouhar2023noiseless} give an
information-theoretic account of what makes a tokenizer good, and
\citet{goldman2024unpacking} and \citet{lotz2025beyond} examine how far
compression predicts downstream quality, the question \S\ref{sec:downstream}
puts to condition C. \citet{schmidt2024tokenization} argue tokenization quality
is not reducible to compression at all and treat pre-tokenization as a
first-class variable, the closest existing precedent to that choice here;
\citet{dagan2024getting} likewise find the pre-tokenization regex measurably
changes downstream performance, for code. That literature treats fertility as
an outcome of vocabulary allocation and corpus composition; our result is
orthogonal to it, a component upstream of both that the corpus-side
interventions it evaluates cannot reach.

Closest to this work, \citet{velayuthan2025egalitarian} identify pre-tokenization
as the source of unfair representation for Tamil, Sinhala, and Hindi and propose
Grapheme Pair Encoding, segmenting on graphemes rather than bytes. We agree on
the diagnosis and differ on what follows from it: they introduce and evaluate a
new segmentation algorithm; we quantify what the existing defect costs under a
one-variable control, derive the bound that explains why the cost resists
corpus composition, provide a diagnostic that identifies the failure without
reading the pre-tokenizer, and measure both the downstream effect and the
defect's prevalence across the deployed ecosystem. GPE changes the segmentation
unit; the repair we measure is a one-character-class edit that leaves
byte-level BPE intact and is already present in \code{o200k}.

\section{Conclusion}

A character class chosen for English places a floor under Nepali tokenization
that Nepali data cannot lift. The floor follows from the pre-tokenizer rather
than from measurement: BPE cannot merge across pre-token boundaries, so a regex
that splits words at every vowel sign bounds fertility before training begins.
On parallel text the effect reaches every abugida we tested and lies dormant in
vocalised Arabic and pointed Hebrew.

The floor is also tight. Across \MpN{} matched pairs spanning
$\MpBoundSpread\times$ in predicted effect, every control arm settles within
\MpMaxBoundGap\% of what a regex match predicts before any training, which makes
the bound a design-time check rather than a post-hoc explanation. Yet the word
class it checks for remains the ecosystem's default, present in
\CensusTgLettersModels\% of the most-downloaded text-generation repositories on
the HuggingFace Hub. Repairing it takes one character class, and the repair was
already known. The diagnostic generalises further than the defect does: a metric
that will not respond to the input that supposedly drives it is constrained
somewhere else, and in a tokenizer pipeline the candidates upstream of the
corpus number three: the normaliser, the pre-tokenizer, and the vocabulary
budget.

\section*{Limitations}
\label{sec:limits}

\paragraph{The repair is prior art.} \code{o200k} already employs a mark-aware
word class. Our contribution lies in measurement, diagnosis, and scope.

\paragraph{\MpN{} trained language pairs, not \NumLanguages{}.} The bound in
\S\ref{sec:scripts} spans \NumLanguages{} languages and requires no training.
The trained pairs cover \MpN{}: \MpLangs. Thai, Khmer, Lao, Myanmar, and Tibetan
are absent from the trained set; they are also the languages where fertility is
undefined and the bound is hardest to interpret. The sweep in
\S\ref{sec:sweep} is Nepali--English only.

\paragraph{Every trained pair is 50/50 with English.} The four new pairs fix the
target-language byte share at 50\%, far above what a general-purpose multilingual
tokenizer would allocate to any single language. That choice keeps the five pairs
comparable to each other and to the Nepali pair, and the sweep confirms the
control arm is insensitive to the share; it does mean the treatment fertilities
exceed what a realistic multilingual mixture would yield, and the
control-to-treatment ratios are correspondingly optimistic.

\paragraph{The census measures exposure.} \S\ref{sec:census} counts
repositories whose pre-tokenizer is letters-only. It does not establish that any
of them is harmed by it: for a model that only ever processes Latin script, the
word class is irrelevant. It is also a snapshot of a live ranking, taken once,
and it excludes gated repositories along with every repository that ships no
\code{tokenizer.json}. The language-tag restriction depends on repository
metadata that authors maintain unevenly.

\paragraph{Translated evaluation text.} FLORES-200 \citep{nllb2022flores} is
translated from English-sourced Wikimedia articles. Its parallelism is what makes
cross-language comparison valid, and it is cleaner and more uniform than native
web text. The harness also supports a secondary evaluation on native FineWeb-2
text.

\paragraph{Whitespace words.} Tokens per whitespace-delimited word is a coarse
measure for morphologically rich languages and undefined for scripts without
inter-word spaces. We report tokens per 100 characters alongside it and omit
fertility where it would mislead.

\paragraph{One seed per downstream condition.} \S\ref{sec:downstream} trains each
of A, B, and C once, at one model scale, and reports bits per byte rather than
task accuracy. The B--C gap on Nepali is \AblBvsCNe\%, too small to separate
from seed variance on a single run; we therefore state that comparison in
compute-normalised form, where the conclusion depends on the sign rather than
the magnitude. Repeated seeds, a second model scale, and downstream task
evaluation would all strengthen the section, and we have not run them.

\paragraph{Fixed vocabulary.} Every comparison holds $|V|$ fixed. The English
cost in \S\ref{sec:pair} follows from that choice and would shrink at a larger
vocabulary. We have not mapped the trade-off.

\section*{Ethical Considerations}

The defect this paper measures is not distributed evenly. Its cost falls
entirely on languages written in abugida scripts and, when diacritics are
present, on Arabic and Hebrew. For Latin, Cyrillic, Hangul, and Han the shatter
ratio is exactly $\EnShatter\times$, so speakers of those languages could not
have noticed the problem from their own usage. Because fertility translates
directly into context length, training cost, and per-token API price, a defect
of this shape widens an existing gap between well-resourced and low-resourced
languages while remaining invisible to the people best placed to correct it.

The failure mode compounds that asymmetry. It presents as a data shortage, the
explanation practitioners working on low-resource languages already expect and
the one they are most likely to act on. Effort then goes into acquiring text
that cannot help, and the misattribution is most expensive where budgets are
smallest.

We see no dual-use concern in this work. The evaluation corpora are public and
contain no personal data. The tokenizer we release is trained on public web text
alongside a private custom corpus that we do not redistribute, and it inherits
whatever biases those sources carry; we make no claim that it is suitable for
deployment without further evaluation. Every experimental result in this paper,
including the matched pair, the sweep, and the downstream ablation, uses only
public data, so none of them depends on the private corpus.

\section*{Acknowledgements}

\paragraph{Use of generative AI.} We disclose this under the ACL Policy on
Publication Ethics. A generative AI assistant was used for literature search,
for generating portions of the experimental code, and for language editing and
sentence restructuring. The authors verified all generated code against the
released harness, checked all cited work directly, and take full responsibility
for the correctness of the methods, results, and writing.

\bibliography{references}

\appendix

\section{The bound in full}
\label{app:bounds}

Table~\ref{tab:boundsfull} is the complete version of Table~\ref{tab:bounds},
adding tokens per 100 characters, defined for every language including the
seven that do not separate words with spaces, and the writing-system family
of each row.

\begin{table*}[h]
\centering
\footnotesize
\setlength{\tabcolsep}{4pt}
\begin{tabular}{llrrrrrr}
\toprule
& & & \multicolumn{2}{c}{Bound, tok/word} & \multicolumn{2}{c}{Bound, tok/100\,char} & \\
\cmidrule(lr){4-5}\cmidrule(lr){6-7}
Language & Script & Shatter & $\backslash$p\{L\}+ & mark & $\backslash$p\{L\}+ & mark & Type \\
\midrule
Thai & Thai & \textbf{9.02$\times$} & -- & -- & 37.4 & 4.1 & abugida \\
Burmese & Myanmar & \textbf{7.58$\times$} & -- & -- & 63.9 & 8.4 & abugida \\
Lao & Lao & \textbf{5.95$\times$} & -- & -- & 39.7 & 6.7 & abugida \\
Malayalam & Malayalam & \textbf{5.82$\times$} & 7.34 & 1.26 & 72.9 & 12.5 & abugida \\
Tamil & Tamil & \textbf{5.66$\times$} & 6.76 & 1.19 & 73.0 & 12.9 & abugida \\
Khmer & Khmer & \textbf{5.11$\times$} & -- & -- & 73.6 & 14.4 & abugida \\
Kannada & Kannada & \textbf{4.86$\times$} & 6.09 & 1.25 & 70.5 & 14.5 & abugida \\
Telugu & Telugu & \textbf{4.57$\times$} & 5.65 & 1.24 & 71.7 & 15.7 & abugida \\
Marathi & Devanagari & \textbf{4.25$\times$} & 4.93 & 1.16 & 70.0 & 16.5 & abugida \\
Odia & Odia & \textbf{4.14$\times$} & 4.81 & 1.16 & 69.3 & 16.8 & abugida \\
Nepali & Devanagari & \textbf{4.09$\times$} & 4.67 & 1.14 & 68.3 & 16.7 & abugida \\
Bengali & Bengali & \textbf{4.03$\times$} & 4.65 & 1.15 & 69.0 & 17.1 & abugida \\
Gujarati & Gujarati & \textbf{3.51$\times$} & 4.03 & 1.15 & 66.0 & 18.8 & abugida \\
Sinhala & Sinhala & \textbf{3.25$\times$} & 4.12 & 1.27 & 64.6 & 19.9 & abugida \\
Hindi & Devanagari & \textbf{2.95$\times$} & 3.30 & 1.12 & 64.0 & 21.7 & abugida \\
Punjabi & Gurmukhi & \textbf{2.87$\times$} & 3.25 & 1.13 & 62.1 & 21.7 & abugida \\
Tibetan & Tibetan & \textbf{1.47$\times$} & -- & -- & 74.4 & 50.6 & abugida \\
\addlinespace
Arabic & Arabic & \textbf{1.08$\times$} & 1.22 & 1.13 & 20.5 & 19.0 & abjad \\
\addlinespace
\emph{5 languages} & Cyrillic, Hangul, Latin & 1.00$\times$ & \multicolumn{4}{c}{\emph{identical under both classes}} & alphabet \\
\addlinespace
\emph{Hebrew} & Hebrew & 1.00$\times$ & \multicolumn{4}{c}{\emph{identical under both classes}} & abjad \\
\addlinespace
\emph{2 languages} & Han, Japanese & 1.00$\times$ & \multicolumn{4}{c}{\emph{identical under both classes}} & logographic \\
\bottomrule
\end{tabular}

\caption{The training-free pre-tokenization bound across FLORES-200 devtest, in
full. Languages at exactly $1.00\times$ have identical bounds under both word
classes and are grouped by family; their per-language values are in the
harness's \code{table\_bounds.csv}.}
\label{tab:boundsfull}
\end{table*}

\section{Corpora, splits and compute budget}
\label{app:repro}

\code{results.json} records the resolved Hub commit of each baseline tokenizer,
the sha256 of each corpus, the version of each library, and the harness's git
commit (\code{\HarnessCommit}). Three baselines load from widely used
republications of a gated upstream, allowing the harness to run without an
account; each is annotated with its source. The FineWeb-2 evaluation slice and
the tokenizer training corpus are partitioned by document index modulo 20,
disjoint at any corpus size. Training corpora are FineWeb-2
\citep{penedo2025fineweb2} for each non-English language and FineWeb-Edu
\citep{penedo2024fineweb} for English; the downstream English validation set is
the validation split of C4 \citep{raffel2020exploring}, distinct from the corpus
those models train on. The census (\S\ref{sec:census}) runs as a separate
command, \code{python -m tokenizer.paper.census}, because it reads a live service
rather than a pinned artifact: rankings and download counts shift, so a rerun
will not reproduce our counts exactly, though every repository we classified is
recorded in \code{census.json}.

\paragraph{Computational budget.} Every tokenizer result in this paper ran on a
laptop CPU: a single 65{,}536-entry BPE over the \PairCorpusMB\,MB corpus takes
one to eight minutes; the \MpN{} matched pairs and the \SweepPoints-point sweep
in both arms amount to 24 tokenizers, roughly three hours total, dominated by
streaming rather than training; the \NumLanguages-language bound and the
vocalisation probe require no training and finish in seconds; the census is
network-bound, \CensusRows{} range requests in roughly fifteen minutes. Only the
downstream evaluation (\S\ref{sec:downstream}) requires a GPU: the three runs
total 14.36B tokens over 268M-parameter models ($2.3\times10^{19}$ FLOPs by the
$6ND$ estimate of \citealp{kaplan2020scaling}) and took \textbf{19.1 GPU-hours on
a single NVIDIA H100 80GB SXM5} (5.3\,h each for A and B, 8.5\,h for C) using a
custom C/CUDA trainer in bf16 with cuDNN fused attention at 34\% of peak FLOPs.

\section{The defect at code-point level}
\label{app:codepoints}

Table~\ref{tab:codepoints} presents the failure at the level where it occurs,
and Table~\ref{tab:vocal} presents the same mechanism switching on in Arabic and
Hebrew once their diacritics are present.

\begin{table}[h]
\centering
\small
\begin{tabular}{llll}
\toprule
Code point & Character & Category & Under \pL{} \\
\midrule
U+0928 & \dev{न}{na}        & \cat{Lo} & kept \\
U+0947 & \dev{े}{-e}      & \cat{Mn} & \textbf{cut} \\
U+092A & \dev{प}{pa}        & \cat{Lo} & kept \\
U+093E & \dev{ा}{-\={a}} & \cat{Mc} & \textbf{cut} \\
U+0932 & \dev{ल}{la}        & \cat{Lo} & kept \\
U+0940 & \dev{ी}{-\={\i}}& \cat{Mc} & \textbf{cut} \\
\bottomrule
\end{tabular}
\caption{The word \devrom{नेपाली}{nep\={a}l\={\i}} (``Nepali'') by code point.
Three of its six characters are marks, so \pL{} produces six single-character
pre-tokens where \pLM{} produces one.}
\label{tab:codepoints}
\end{table}

\begin{table}[h]
\centering
\small
\begin{tabular}{lrrr}
\toprule
Sample & Marks & Unvocalised & Vocalised \\
\midrule
Arabic & 33 & 1.00$\times$ & \textbf{6.50$\times$} \\
Hebrew & 23 & 1.00$\times$ & \textbf{6.14$\times$} \\
Devanagari & 22 & 1.00$\times$ & \textbf{5.22$\times$} \\
\bottomrule
\end{tabular}

\caption{Shatter ratio on the same sentence with its combining marks stripped
(``unvocalised'') and intact (``vocalised'').}
\label{tab:vocal}
\end{table}

\section{The Nepali pair in bytes per token}
\label{app:pair}

Table~\ref{tab:multipair} reports the headline pair in tokens per word, the unit
the rest of the paper uses. Table~\ref{tab:pair} supplements this with bytes per
token for the same two tokenizers, the quantity that converts to bits per byte in
\S\ref{sec:downstream} and the reason the downstream arms read different amounts
of text at the same token budget.

\begin{table}[h]
\centering
\small
\setlength{\tabcolsep}{3pt}
\begin{tabular}{lrrrr}
\toprule
& \multicolumn{2}{c}{Nepali} & \multicolumn{2}{c}{English} \\
\cmidrule(lr){2-3}\cmidrule(lr){4-5}
Word class & tok/word & B/tok & tok/word & B/tok \\
\midrule
Control & 4.78 & 3.81 & 1.28 & 4.76 \\
Treatment & \textbf{1.58} & 11.53 & 1.30 & 4.67 \\
\midrule
Relative change & $-$66.9\% & $+$202.5\% & +2.12\% & -2.08\% \\
\bottomrule
\end{tabular}

\caption{The controlled comparison on FLORES-200 devtest. Corpus
\PairCorpusMB\,MB at \PairNepaliFrac{} Nepali by bytes, vocabulary \PairVocab{},
identical in every other respect.}
\label{tab:pair}
\end{table}

\section{The mixture sweep, in numbers}
\label{app:sweep}

Table~\ref{tab:sweep} contains the data behind Figure~\ref{fig:sweep}. Each row
corresponds to two tokenizers trained from scratch on the same corpus, differing
in one character class.

\begin{table*}[h]
\centering
\footnotesize
\setlength{\tabcolsep}{3.5pt}
\begin{tabular}{lrrrrrrr}
\toprule
Nepali share of corpus & 5\% & 20\% & 35\% & 50\% & 65\% & 80\% & 95\% \\
\midrule
Control & 4.84 & 4.79 & 4.78 & 4.78 & 4.77 & 4.76 & 4.76 \\
\quad relative spread & \multicolumn{7}{c}{1.7\%} \\
Treatment & 1.97 & 1.70 & 1.63 & 1.58 & 1.54 & 1.51 & 1.47 \\
\quad relative spread & \multicolumn{7}{c}{33.9\%} \\
\midrule
Pre-tokenization bound & \multicolumn{7}{c}{4.67 vs 1.14} \\
\bottomrule
\end{tabular}

\caption{The mixture sweep. Nepali fertility on FLORES-200 devtest, by Nepali
share of the tokenizer training corpus. Control is \pL{}, treatment is \pLM{};
the two arms are identical in every other respect. The last row is the
training-free bound of \S\ref{sec:bound} for each arm.}
\label{tab:sweep}
\end{table*}

\section{The distinct patterns behind the census}
\label{app:census}

Table~\ref{tab:censuspat} enumerates the most common distinct pre-tokenization
patterns across the repositories of \S\ref{sec:census}, with an example
repository for each. Two of them agree for their first sixty characters and
diverge only in a quantifier further right, so the pattern text alone does not
identify a row. The letters-only rows are not the work of anyone who chose a
letters-only word class. They are GPT-2's pattern and its direct descendants,
reaching those repositories because \code{ByteLevel} applies it whenever
\code{use\_regex} is left at its default, which is the sense in which this is an
inherited default rather than a decision.

\begin{table}[h]
\centering
\small
\begin{tabular}{rlp{0.46\linewidth}}
\toprule
Repos & Word class & Pattern (truncated) \\
\midrule
664 & not letter-based & \emph{no split stage that can cut inside a word}\newline{\scriptsize\itshape e.g.\ \seqsplit{sentence-transformers/all-MiniLM-L6-v2}} \\
510 & letters only & \texttt{\scriptsize\seqsplit{(?i:'s|'t|'re|'ve|'m|'ll|'d)|[\textasciicircum{}\textbackslash{}r\textbackslash{}n\textbackslash{}p\{L\}\textbackslash{}p\{N\}]?\textbackslash{}p\{L\}+|\textbackslash{}p\{N\}| ?[\textasciicircum{}\ldots}}\newline{\scriptsize\itshape e.g.\ \seqsplit{Qwen/Qwen3-0.6B}} \\
325 & letters only & \texttt{\scriptsize\seqsplit{'s|'t|'re|'ve|'m|'ll|'d| ?\textbackslash{}p\{L\}+| ?\textbackslash{}p\{N\}+| ?[\textasciicircum{}\textbackslash{}s\textbackslash{}p\{L\}\textbackslash{}p\{N\}]+|\textbackslash{}s+\ldots}}\newline{\scriptsize\itshape e.g.\ \seqsplit{openai/clip-vit-base-patch32}} \\
246 & letters only & \texttt{\scriptsize\seqsplit{(?i:'s|'t|'re|'ve|'m|'ll|'d)|[\textasciicircum{}\textbackslash{}r\textbackslash{}n\textbackslash{}p\{L\}\textbackslash{}p\{N\}]?\textbackslash{}p\{L\}+|\textbackslash{}p\{N\}\{1,3\}\ldots}}\newline{\scriptsize\itshape e.g.\ \seqsplit{zai-org/GLM-OCR}} \\
108 & mark-aware & \texttt{\scriptsize\seqsplit{(?i:'s|'t|'re|'ve|'m|'ll|'d)|[\textasciicircum{}\textbackslash{}r\textbackslash{}n\textbackslash{}p\{L\}\textbackslash{}p\{N\}]?[\textbackslash{}p\{L\}\textbackslash{}p\{M\}]+|\textbackslash{}p\{\ldots}}\newline{\scriptsize\itshape e.g.\ \seqsplit{Qwen/Qwen3.5-9B}} \\
60 & mark-aware & \texttt{\scriptsize\seqsplit{[\textasciicircum{}\textbackslash{}r\textbackslash{}n\textbackslash{}p\{L\}\textbackslash{}p\{N\}]?[\textbackslash{}p\{Lu\}\textbackslash{}p\{Lt\}\textbackslash{}p\{Lm\}\textbackslash{}p\{Lo\}\textbackslash{}p\{M\}]*[\textbackslash{}p\{Ll\}\textbackslash{}p\{Lm\}\textbackslash{}\ldots}}\newline{\scriptsize\itshape e.g.\ \seqsplit{nvidia/Nemotron-3-Nano-Omni-30B-A3B-Reasoning-NVFP4}} \\
\bottomrule
\end{tabular}

\caption{The most common distinct pre-tokenization patterns in the census,
ranked by number of repositories. Patterns are truncated for width; the full
text of each, along with an example repository, is in
\code{table\_census\_pat.csv}.}
\label{tab:censuspat}
\end{table}

Table~\ref{tab:patterns} applies the same classifier to the thirteen released
tokenizers of \S\ref{sec:cross}, for which a measured Nepali fertility is
available for each. The table checks that the classifier's labels correspond to
behaviour and not merely to syntax: the letters-only group occupies the top of
the table, and nothing below it is letters-only.

\begin{table*}[h]
\centering
\small
\setlength{\tabcolsep}{4pt}
\begin{tabular}{llrr}
\toprule
Tokenizer & Pre-tokenization word class & $|V|$ & Nepali tok/word \\
\midrule
Llama-3 & letters-only & 128,256 & 3.76 \\
Control (\textbackslash{}p\{L\}+) & letters-only & 65,536 & 4.78 \\
Qwen2.5 & letters-only & 151,665 & 6.57 \\
cl100k (GPT-3.5/4) & letters-only & 100,277 & 6.98 \\
GPT-2 & letters-only & 50,257 & 10.97 \\
\addlinespace
Treatment ([\textbackslash{}p\{L\}\textbackslash{}p\{M\}]+) & mark-aware & 65,536 & 1.58 \\
\sysname{} & mark-aware & 65,536 & 1.69 \\
o200k (GPT-4o) & mark-aware & 200,019 & 2.32 \\
Mistral NeMo (Tekken) & mark-aware & 131,072 & 3.17 \\
DeepSeek-V3 & mark-aware & 128,815 & 4.29 \\
\addlinespace
IndicBERTv2 & not-letter-based & 250,000 & 1.58 \\
BLOOM & not-letter-based & 250,680 & 1.72 \\
NLLB-200 & not-letter-based & 256,204 & 1.92 \\
mT5 & not-letter-based & 250,100 & 2.64 \\
Sarvam-1 & not-letter-based & 68,096 & 2.66 \\
Gemma-2 & not-letter-based & 256,000 & 3.13 \\
\bottomrule
\end{tabular}

\caption{Word class and measured Nepali fertility for each tokenizer in
Table~\ref{tab:main}, sorted by class and then by fertility.}
\label{tab:patterns}
\end{table*}

Two classes of repository are excluded from each denominator in
Table~\ref{tab:census}: those that ship no tokenizer of any kind
(\CensusExcludedNoTok{} of the overall top-\CensusLimit{} listing, comprising
vision, audio, and diffusion models, plus quantised redistributions that carry
weights only), and gated repositories alongside the handful whose
\code{pre\_tokenizer} did not fit in an 8\,MB range request. Neither is folded
into the affected or the unaffected side, because doing so would shift the
headline number with no evidence behind it.

What is deliberately \emph{not} excluded is the group that ships a tokenizer in
some older or non-\code{tokenizers} format. Excluding those would bias the
result twice over and in opposite directions: SentencePiece repositories are
definitively unaffected, so dropping them raises the letters-only share, while
repositories carrying \code{vocab.json} and \code{merges.txt} without a
\code{tokenizer.json} are definitively affected, since that pair is GPT-2's
original serialisation and \code{transformers} loads it through \code{ByteLevel}
with the regex enabled, so dropping them lowers it. Both are classified from
their file layout, which resolves the question without downloading anything.

\section{Predicted against measured, across five languages}
\label{app:multipair}

Figure~\ref{fig:multipair} plots each arm's own training-free bound
(\S\ref{sec:bound}) against the fertility that arm actually attains, for all
\MpN{} matched pairs of \S\ref{sec:pair} and both arms of each.
Proposition~\ref{prop:bound} forces every one of the ten points onto or above the
dashed line and says nothing further. What the theorem does not force is the
difference between the two series: the control arm lies on the line across the
entire range, while the treatment arm lies well above its own much lower floor.
That difference is the practical content of the bound. Where the letters-only
class is in force, the pre-tokenizer determines the outcome and corpus-side
interventions cannot reach it; where it is not, the floor is slack and the
vocabulary budget becomes the variable a designer can trade against.

\begin{figure}[h]
\centering
\resizebox{\linewidth}{!}{\begin{tikzpicture}
\begin{axis}[width=0.84\linewidth, height=6.4cm,
  xlabel={Pre-tokenization bound for that arm (tokens/word)},
  ylabel={Measured fertility (tokens/word)},
  xmin=0.89, xmax=9.21, ymin=0.89, ymax=8.08,
  grid=major, legend pos=north west, legend cell align=left]
\addplot[dashed, gray, domain=0.89:8.08, samples=2] {x};
\addlegendentry{$y=x$: the bound}
\addplot[only marks, mark=*] coordinates {(3.2993,3.3597) (4.6454,4.7295) (4.6726,4.7753) (6.7556,6.8482) (7.3404,7.4831)};
\addlegendentry{control \pL{}}
\addplot[only marks, mark=square*] coordinates {(1.1171,1.3216) (1.1519,1.4955) (1.1415,1.5786) (1.1943,1.7745) (1.2608,2.1116)};
\addlegendentry{treatment \pLM{}}
\node[anchor=north west, font=\scriptsize] at (axis cs:3.2993,3.3597) {\,Hindi};
\node[anchor=north west, font=\scriptsize] at (axis cs:4.6454,4.7295) {\,Bengali};
\node[anchor=north west, font=\scriptsize] at (axis cs:4.6726,4.7753) {\,Nepali};
\node[anchor=north west, font=\scriptsize] at (axis cs:6.7556,6.8482) {\,Tamil};
\node[anchor=north west, font=\scriptsize] at (axis cs:7.3404,7.4831) {\,Malayalam};
\end{axis}
\end{tikzpicture}}
\caption{Each arm's pre-tokenization bound against the fertility it reaches on
FLORES-200 devtest, two points per matched pair. The dashed line is $y=x$, the
bound. Proposition~\ref{prop:bound} guarantees no point falls below it. Control
points lie on the line; treatment points lie well above their own floor.}
\label{fig:multipair}
\end{figure}

\end{document}